\def\year{2017}\relax
\documentclass[letterpaper]{article}
\usepackage{aaai17}
\usepackage{times}
\usepackage{helvet}
\usepackage{courier}
\usepackage{multirow}

\usepackage{mathtools}

\usepackage{graphicx} 
\usepackage{subfig}
\usepackage{epstopdf}
\usepackage{algorithm}
\usepackage[noend]{algorithmic}
\usepackage{tabularx}
\usepackage{amsfonts}
\usepackage{amsmath}
\usepackage{pgfplots}
\usetikzlibrary{backgrounds}
\pgfplotsset{compat=newest} 
\pgfplotsset{plot coordinates/math parser=false} 
\usepgfplotslibrary{external}
\usepackage{amsthm}
\usepackage{float}
\usepackage{url}
\usepackage{thmtools, thm-restate}
\usepackage{rotating}

\newtheorem{theorem}{Theorem}
\newtheorem*{theorem*}{Theorem}
\newtheorem{definition}{Definition}
\newtheorem{proposition}{Proposition}
\newtheorem{remark}{Remark}

\newcolumntype{Y}{>{\centering\arraybackslash}X}
\newcolumntype{C}[1]{>{\centering\let\newline\\\arraybackslash\hspace{0pt}}m{#1}}
\newcolumntype{L}[1]{>{\raggedright\let\newline\\\arraybackslash\hspace{0pt}}m{#1}}

\newcommand{\NAME}{UTS}
\newcommand{\eps}{\varepsilon}

\newenvironment{sketchproof}{%
\proof}{\endproof}

\begin{document}

\title{Unimodal Thompson Sampling for Graph--Structured Arms}
\author{Stefano Paladino and Francesco Trov\`o and Marcello Restelli and Nicola Gatti\\
Dipartimento di Elettronica, Informazione e Bioingegneria\\
Politecnico di Milano, Milano, Italy\\
\{stefano.paladino, francesco1.trovo, marcello.restelli, nicola.gatti\}@polimi.it
}

\maketitle

\begin{abstract}
We study, to the best of our knowledge, the first Bayesian algorithm for unimodal Multi--Armed Bandit (MAB) problems with graph structure. In this setting, each arm corresponds to a node of a graph and each edge provides a relationship, unknown to the learner, between two nodes in terms of expected reward. Furthermore, for any node of the graph there is a path leading to the unique node providing the maximum expected reward, along which the expected reward is monotonically increasing. Previous results on this setting describe the behavior of frequentist MAB algorithms. In our paper, we design a Thompson Sampling--based algorithm whose asymptotic pseudo--regret matches the lower bound for the considered setting. We show that---as it happens in a wide number of scenarios---Bayesian MAB algorithms dramatically outperform frequentist ones. In particular, we provide a thorough experimental evaluation of the performance of our and state--of--the--art algorithms as the properties of the graph vary.
\end{abstract}

\section{Introduction}
Multi--Armed Bandit (MAB) algorithms~\cite{auer2002finite} have been proven to provide effective solutions for a wide range of applications fitting the sequential decisions making scenario. In this framework, at each round over a finite horizon $T$, the learner selects an action (usually called \emph{arm}) from a finite set and observes only the reward corresponding to the choice she made. The goal of a MAB algorithm is to converge to the optimal arm, i.e., the one with the highest expected reward, while minimizing the loss incurred in the learning process and, therefore, its performance is measured through its expected \emph{regret}, defined as the difference between the expected reward achieved by an oracle algorithm always selecting the optimal arm and the one achieved by the considered algorithm. We focus on the so--called \emph{Unimodal} MAB (UMAB), introduced in~\cite{combes2014unimodala}, in which each arm corresponds to a node of a graph and each edge is associated with a relationship specifying which node of the edge gives the largest expected reward (providing thus a partial ordering over the arm space). Furthermore, from any node there is a path leading to the unique node with the maximum expected reward along which the expected reward is monotonically increasing. While the graph structure may be (not necessarily) known \emph{a priori} by the UMAB algorithm, the relationship defined over the edges is discovered during the learning. In the present paper, we propose a novel algorithm relying on the Bayesian learning approach for a generic UMAB setting.

Models presenting a graph structure have become more and more interesting in last years due to the spread of social networks. Indeed, the relationships among the entities of a social network have a natural graph structure. A practical problem in this scenario is the \emph{targeted advertisement} problem, whose goal is to discover the part of the network that is interested in a given product. This task is heavily influenced by the graph structure, since in social networks people tend to have similar characteristics to those of their friends (i.e., neighbor nodes in the graph), therefore interests of people in a social network change smoothly and neighboring nodes in the graph look similar to each other~\cite{mcpherson2001birds,crandall2008feedback}. More specifically, an advertiser aims at finding those users that maximize the ad expected revenue (i.e., the product between click probability and value per click), while at the same time reducing the amount of times the advertisement is presented to people not interested in its content. 

Under the assumption of unimodal expected reward,  the learner can move from low expected rewards to high ones just by climbing them in the graph, preventing from the need of a uniform exploration over all the graph nodes. This assumption reduces the complexity in the search for the optimal arm, since the learning algorithm can avoid to pull the arms corresponding to some subset of non--optimal nodes, reducing thus the regret. Other applications might benefit from this structure, e.g., \emph{recommender systems} which aims at coupling items with those users are likely to enjoy them. Similarly, the use of the unimodal graph structure might provide more meaningful recommendations without testing all the users in the social network. Finally, notice that unimodal problems with a single variable, e.g., in sequential pricing~\cite{jia2011unimodal}, bidding in online sponsored search auctions~\cite{edelman2007strategic} and single--peak preferences economics and voting settings~\cite{mas1995micreconomic}, are graph--structured problems in which the graph is a line.
 
Frequentist approaches for UMAB with graph structure are proposed in~\cite{jia2011unimodal} and~\cite{combes2014unimodala}. Jia and Mannor~\shortcite{jia2011unimodal} introduce the GLSE algorithm with a regret of order $O(\sqrt{T}\log(T))$. However, GLSE performs better than classical bandit algorithms only when the number of arms is $\Theta(T)$. Combes and Proutiere~\shortcite{combes2014unimodala} present the OSUB algorithm---based on KLUCB---achieving asymptotic regret of $O(\log(T))$ and outperforming GLSE in settings with a few arms. To the best of our knowledge, no Bayesian approach has been proposed for unimodal bandit settings, included the UMAB setting we study. However, it is well known that Bayesian MAB algorithms---the most popular is Thompson Sampling (TS)---usually suffer of same order of regret as the best frequentist one (e.g., in unstructured settings~\cite{kaufmann2012thompson}), but they outperform the frequentist methods in a wide range of problems (e.g., in bandit problems without structure~\cite{chapelle2011empirical} and in bandit problems with budget~\cite{xia2015thompson}). Furthermore, in problems with structure, the classical Thompson Sampling (not exploiting the problem structure) may outperform frequentist algorithms exploiting the problem structure. For this reason, in this paper we explore Bayesian approaches for the UMAB setting. More precisely, we provide the following original contributions:
\begin{itemize}
	\item we design a novel Bayesian MAB algorithm, called \NAME{} and based on the TS algorithm;
	\item we derive a tight upper bound over the pseudo--regret for \NAME{}, which asymptotically matches the lower bound for the UMAB setting; 
	\item we describe a wide experimental campaign showing better performance of \NAME{} in applicative scenarios than those of state--of--the--art algorithms, evaluating also how the performance of the algorithms (ours and of the state of the art) varies as the graph structure properties vary.
\end{itemize}


\section{Related work}
Here, we mention the main works related to ours. Some works deal with unimodal reward functions in continuous armed bandit setting~\cite{jia2011unimodal,combes2014unimodalb,kleinberg2008multi}. In~\cite{jia2011unimodal} a successive elimination algorithm, called LSE, is proposed achieving regret of $O(\sqrt{T} \log{T})$. In this case, assumptions over the minimum local decrease and increase of the expected reward is required. Combes and Proutiere~\shortcite{combes2014unimodalb} consider stochastic bandit problems with a continuous set of arms and where the expected reward is a continuous and unimodal function of the arm. They propose the SP algorithm, based on the stochastic pentachotomy procedure to narrow the search space. Unimodal MABs on metric spaces are studied in~\cite{kleinberg2008multi}.

An application--dependent solution to the recommendation systems which exploits the similarity of the graph in social network in targeted advertisement has been proposed in~\cite{valko2014spectral}. Similar information has been considered in~\cite{caron2013mixing} where the problem of cold--start users (i.e., new users) is studied. Another type of structure considered in sequential games is the one of \emph{monotonicity} of the conversion rate in the price~\cite{trovo2015multiarmed}. Interestingly, the assumptions of monotonicity and unimodality are orthogonal, none of them being a special case of the other, therefore the results for monotonic setting cannot be used in unimodal bandits. In~\cite{alon2013bandits,mannor2011bandits}, a graph structure of the arm feedback in an \emph{adversarial} setting is studied. More precisely, they assume to have correlation over rewards and not over the expected values of arms.

\section{Problem Formulation}
A learner receives in input a finite undirected graph MAB setting $G = (A,E)$, whose vertices $A = \{a_1, \ldots, a_K \}$ with $K \in \mathbb{N}$ correspond to the arms and an edge $(a_i a_j) \in E$ exists only if there is a direct partial order relationship between the expected rewards of arms $a_i$ and $a_j$. The leaner knows \emph{a priori} the nodes and the edges (i.e., she knows the graph), but, for each edge, she does not know \emph{a priori} which is the node of the edge with the largest expected reward (i.e., she does not know the ordering relationship). At each round~$t$ over a time horizon of $T \in \mathbb{N}$ the learner selects an arm~$a_i$ and gains the corresponding reward $x_{i,t}$. This reward is drawn from an i.i.d.~random variable $X_{i,t}$ (i.e., we consider a stochastic MAB setting) characterized by an unknown distribution $\mathcal{D}_i$ with finite known support $\Omega \subset \mathbb{R}$ (as customary in MAB settings, from now on we consider $\Omega \subseteq [0,1]$) and by unknown expected value $\mu_i := \mathbb{E}[X_{i,t}]$. We assume that there is a single optimal arm, i.e., there exists a unique arm $a_{i^*}$ s.t.~its expected value $\mu_{i^*} = \max_i \mu_i$ and, for sake of notation, we denote $\mu_{i^*}$ with $\mu^*$.

Here we analyze a graph bandit setting with unimodality property, defined as:
\begin{definition}
	A graph \emph{unimodal MAB} (UMAB) setting $G = (A,E)$ is a graph bandit setting $G$ s.t.~for each sub--optimal arm $a_i, i \neq i^*$ it exists a finite path $p = (i_1 = i, \ldots, i_m = i^*)$ s.t.~$\mu_{i_k} < \mu_{i_{k+1}}$ and $(a_{i_k}, a_{i_{k+1}}) \in E$ for each $k \in \{1, \ldots, m-1\}$. 
\end{definition}
This definition assures that if one is able to identify a non--decreasing path in $G$ of expected rewards, she be able to reach the optimum arm, without getting stuck in local optima. Note that the unimodality property implies that the graph $G$ is connected and therefore we consider only connected graphs from here on.

A policy $\mathfrak{U}$ over a UMAB setting is a procedure able to select at each round~$t$ an arm~$a_{i_t}$ by basing on the history~$h_t$, i.e., the sequence of past selected arms and past rewards gained. The pseudo--regret $R_T(\mathfrak{U})$ of a generic policy $\mathfrak{U}$ over a UMAB setting is defined as:
\begin{equation}\label{pseudoregretdefinition}
	R_T(\mathfrak{U}) := T \mu^* - \mathbb{E}\left[ \sum_{t = 1}^T X_{i_t,t} \right],
\end{equation}
where the expected value $\mathbb{E}[\cdot]$ is taken w.r.t.~the stochasticity of the gained rewards $X_{i_t,t}$ and of the policy $\mathfrak{U}$.

Let us define the neighborhood of arm $a_i$ as $N(i) := \{j | (a_i a_j) \in E\}$, i.e., the set of each index $j$ of the arm~$a_j$ connected to the arm~$a_i$ by an edge $(a_i a_j) \in E$. It has been shown in~\cite{combes2014unimodala} that the problem of learning in a UMAB setting presents a lower bound over the regret $R_T(\mathfrak{U})$ of the following form:
\begin{theorem} \label{teo:lower}
	Let $\mathfrak{U}$ be a uniformly good policy, i.e., a policy s.t.~$R_T(\mathfrak{U}) = o(T^c)$ for each $c > 0$. Given a UMAB setting $G = (A,E)$ we have:
	\begin{equation}\label{equ:lb}
		\liminf_{T \rightarrow \infty} \frac{R_T(\mathfrak{U})}{\log(T)} = \sum_{i \in N(i^*)} \frac{\mu^* - \mu_i}{KL(\mu_i,\mu^*)}
	\end{equation}
	where $KL(p,q) = p \log \left(\frac{p}{q} \right) + (1 - p) \log \left(\frac{1-p}{1-q} \right)$, i.e., the Kullaback--Leibler divergence of two Bernoulli distributions with means $p$ and $q$, respectively.
\end{theorem}
This result is similar to the one provided in~\cite{lai1985asymptotically}, with the only difference that the summation is restricted to the arms laying in the neighborhood of the optimal arm $N(i^*)$ and reduces to it when the optimal arm is connected to all the others (i.e., $N(i^*) \equiv \{1, \ldots, K\}$) or the graph is completely connected (i.e., $N(i) \equiv \{1, \ldots, K\}, \forall i$). We would like to point out that by relying on the assumption of having a single maximum of the expected rewards, we also assure that the optimal arm neighborhood $N(i^*)$ is uniquely defined and, thus, the lower bound inequality in Equation~\ref{equ:lb} is well defined.

\section{The \NAME{} algorithm}

We describe the \NAME{} algorithm and we show that its regret is asymptotically optimal, i.e., it asymptotically matches the lower bound of Theorem~\ref{teo:lower}. The algorithm is an extension of the Thompson Sampling~\cite{thompson1933likelihood} that exploits the graph structure and the unimodal property of the UMAB setting. Basically, the rationale of the algorithm is to apply a simple variation of the TS algorithm to only the arms associated with the nodes that compose the neighborhood of the arm with the highest empirical mean reward, called \emph{leader}.

\subsection{The \NAME{} pseudo--code}

\begin{algorithm}[ht]
\caption{\NAME{}}
\begin{algorithmic}[1]
\STATE {\bf Input:} UMAB setting $G = (V,E)$, Horizon $T$, Priors $\{ \pi_i\}_{i=1}^K$\label{line:setting}
\FOR{$t \in \{ 1, \ldots, T \}$}
	\STATE Compute $\hat{\mu}_{i, T_{i,t}}$ for each $i \in \{1, \ldots, K\}$ \label{line:emp_mean}
	\STATE Find the leader $a_{l(t)}$ 
	\IF{$L_{l(t),t} \bmod |N^+(l(t))| = 0$}
		\STATE Collect reward $x_{l(t),t}$ \label{line:pullleader}
	\ELSE
		\STATE Draw $\theta_{i,t}$ from $\pi_{i,t}$ for each $i \in N^+(l(t))$ \label{line:init_ts}
		\STATE Collect reward $x_{i_t,t}$ where $i_t = \arg \max_i \theta_{i,t}$ \label{line:end_ts}
	\ENDIF
\ENDFOR
\end{algorithmic}
\label{alg:UTS}
\end{algorithm}

The pseudo--code of the \NAME{} algorithm is presented in Algorithm~\ref{alg:UTS}. The algorithm receives in input the graph structure $G$, the time horizon $T$, and a Bayesian prior $\pi_i$ for each expected reward $\mu_i$. At each round $t$, the algorithm computes the empirical expected reward for each arm (Line~\ref{line:emp_mean}):
\begin{equation*}
	\hat{\mu}_{i,t} := \left\{ 
		\begin{array}{lr}
			\dfrac{S_{i_t}}{T_{i,t}} & \text{if } T_{i,t} > 0 \\
			0 & \text{otherwise} 
		\end{array} \right.,
\end{equation*}
where $S_{i,t} = \sum_{h = 1}^{t-1} X_{i,h} \mathbf{1} \{\mathfrak{U}(h) = a_{i} \}$ is the cumulative reward of arm $a_i$ up to round $t$ and $T_{i,t} = \sum_{h = 1}^{t-1} \mathbf{1} \{\mathfrak{U}(h) = a_{i} \}$ is the number of times the arm $a_i$ has been pulled up to round $t$.\footnote{We here denote with $\mathbf{1}\{ \cdot \}$ the indicator function.} After that, \NAME{} selects the arm denoted as the leader $a_{l(t)}$ for round $t$, i.e., the one having the maximum empirical expected reward:
\begin{equation}
	a_{l(t)} = \arg \max_{a_i \in A} \hat{\mu}_{i,t}.
\end{equation}
Once the leader has been chosen, we restrict the selection procedure to it and its neighborhood, considering only arms with indexes in $N^+(l(t)) := N(l(t)) \cup \{l(t)\}$. Denote with $L_{i,t} := \sum_{h = 1}^{t-1} \mathbf{1}\{l(h) = i\}$ the number of times the arm~$a_i$ has been selected as leader before round $t$. If $L_{l(t),t}$ is a multiple of $|N^+(l(t))|$, then the leader is pulled and reward~$x_{l(t),t}$ is gained (Line~\ref{line:pullleader}).\footnote{We here denote with $|\cdot|$ the cardinality operator.} Otherwise, the TS algorithm is performed over arms $a_i$ s.t.~$i \in N^+(l(t))$ (Lines~\ref{line:init_ts}--\ref{line:end_ts}).

Basically, under the assumption of having a prior $\pi_{i}$, we can compute the posterior distribution $\pi_{i,t}$ for $\mu_i$ after $t$ rounds, using the information gathered from the rounds in which $a_i$ has been pulled. We denote with $\theta_{i,t}$ a sample drawn from $\pi_{i,t}$, called \emph{Thompson sample}. For instance, for Bernoulli rewards and by assuming uniform priors we have that $\pi_{i,t} = Beta(1 + S_{i,t}, 1 + T_{i,t} - S_{i,t})$, where $Beta(\alpha, \beta)$ is the beta distribution with parameters $\alpha$ and $\beta$. Finally, the \NAME{} algorithm pulls the arm with the largest Thompson sample $\theta_{i,n}$ and collects the corresponding reward $x_{i_t,t}$. See~\cite{kaufmann2012thompson} for further details.

\begin{remark}
	Assuming that the \NAME{} algorithm receives in input the whole graph $G$ is unnecessary. The algorithm just requires an oracle that, at each round $t$, is able to return the neighborhood $N(l(t))$ of the arm which is currently the leader $a_{l(t)}$. This is crucial in all the applications in which the graph is discovered by means of a series of queries and the queries have a non--negligible cost (e.g., in social networks a query might be computationally costly). Finally, we remark that the frequentist counterpart of our algorithm (i.e., the OSUB algorithm) requires the computation of the maximum node degree $\gamma := \max_i |N(i)|$, thus requiring at least an initial analysis of the entire graph $G$.
\end{remark}

\subsection{Finite--time analysis of \NAME{}}

\begin{restatable}{theorem}{teouts} \label{teo:uts}
Given a UMAB setting $G = (A,E)$, the expected pseudo--regret of the \NAME{} algorithm satisfies, for every $\eps > 0$:

\begin{scriptsize}
\begin{align*}
\mathcal{R}_T(\text{\NAME{}}) \leq (1 + \eps) \sum_{i \in N(i^*)} \frac{\mu^* - \mu_i}{KL(\mu_i,\mu^*)} [\log(T) + \log\log(T)] + \tilde{C},
\end{align*}
\end{scriptsize}

\noindent where $\tilde{C} > 0$ is a constant depending on $\eps$, the number of arms $K$ and the expected rewards $\{\mu_1, \ldots, \mu_K\}$.
\end{restatable}

\begin{sketchproof}
(The complete version of the proof is reported in the appendices.) At first, we remark that a straightforward application of the proof provided for OSUB is not possible in the case of \NAME{}. Indeed, the use of frequentist upper bounds over the expected reward in OSUB implies that in finite time and with high probability the bounds are ordered as the expected values. Since we are using a Bayesian algorithm, we would require the same assurance over the Thompson samples $\theta_{i,t}$, but we do not have a direct bound over $\mathbb{P} (\theta_{i,t} > \theta_{i',t})$ where $a_{i'}$ is the optimal arm in the neighborhood $N^+(i)$. This fact requires to follow a completely different strategy when we analyze the case in which the leader is not the optimal arm.

The regret of the \NAME{} algorithm $R_T(\text{\NAME{}})$ can be divided in two parts: the one obtained during those rounds in which the optimal arm $a^*$ is the leader, called $\mathcal{R}_1$, and the summation of the regrets in the rounds in which the leader is the arm $a_i \neq a^*$, called $\mathcal{R}_i$. $\mathcal{R}_1$ is obtained when $i^*$ is the leader, thus, the \NAME{} algorithms behaves like Thompson Sampling restricted to the optimal arm and its neighborhood $N^+(i^*)$, and the regret upper bound is the one derived in~\cite{kaufmann2012thompson} for the TS algorithm.

$\mathcal{R}_i$ is upper bounded by the expected number of rounds the arm $a_i$ has been selected as leader $\mathbb{E}[L_{i,T}]$ over the horizon $T$. Let us consider $\hat{L}_{i,T}$ defined as the number of rounds spent with $a_i$ as leader when restricting the problem to its neighborhood $N^+(i)$. $\mathbb{E}[\hat{L}_{i,T}]$ is an upper bound over $\mathbb{E}[L_{i,T}]$, since there is nonzero probability that the \NAME{} algorithm moves in another neighborhood. Since $i \neq i^*$ and the setting is unimodal, there exists an optimal arm $a_{i'}, i' \neq i$ among those in the neighborhood $N(i)$ s.t.~$\mu_{i'} = \max_{i | a_i \in N(i)} \mu_i$ and $\hat{\mu}_{i,t} \geq \hat{\mu}_{i'}$. Thus:

{ \fontsize{8.5pt}{8pt}
\begin{align*}
	&\ \mathcal{R}_i \leq \mathbb{E}[\hat{L}_{i,T}] = \sum_{t = 1}^T \mathbb{E} \left[ \mathbf{1} \lbrace \hat{\mu}_{i,t} = \max_{a_j \in N^+(i)} \hat{\mu}_{j,t} \rbrace \right] \\
	& = \sum_{t = 1}^T \mathbb{P} \left( \hat{\mu}_{i,t} \geq \max_{a_j \in N^+(i)} \hat{\mu}_{j,t} \right) \leq \sum_{t = 1}^T \mathbb{P} \left( \hat{\mu}_{i,t} \geq \hat{\mu}_{i',t} \right) \\
	& = \sum_{t = 1}^T \mathbb{P} \left( \hat{\mu}_{i,t} - \mu_i - \frac{\Delta_i}{2} - \hat{\mu}_{i',t} + \mu_{i'} - \frac{\Delta_i}{2} \geq 0 \right) \\
	& \leq \underbrace{\sum_{t = 1}^T \mathbb{P} \left( \hat{\mu}_{i,t} - \mu_i - \frac{\Delta_i}{2} \geq 0 \right)}_{\mathcal {R}_{i1}} + \underbrace{\sum_{t = 1}^T \mathbb{P} \left( \hat{\mu}_{i',t} - \mu_{i'} + \frac{\Delta_i}{2} \leq 0 \right)}_{\mathcal {R}_{i2}}, \\
\end{align*}
}

\noindent where $\Delta_i = \max_{i' |a_i \in N(i)} \mu_{i'} - \mu_i$ is the expected loss incurred in choosing $a_i$ instead of its best adjacent one $a_{i'}$.

$\mathcal{R}_{i1}$ can be upper bounded by a constant by relying on conditional probability definition and the Hoeffding inequality~\cite{hoeffding1963probability}. Specifically, we rely on the fact that the leader is chosen at least $\left\lfloor \frac{L_{l(t),t}}{|N^+(l(t))|} \right\rfloor$ times. Upper bounding $\mathcal{R}_{i2}$ by a constant term requires the use of Proposition 1 in~\cite{kaufmann2012thompson}, which limits the expected number of times the optimal arm is pulled less than $t^b$ times by TS, where $b \in (0,1)$ is a constant, and the use of a technique already used on $\mathcal{R}_{i1}$. Summing up the regret over $i \neq i^*$ and considering the three obtained bounds concludes the proof.
\end{sketchproof}

\begin{table}[H]
\centering
\caption{Results concerning $R_{\%}(\mathfrak{U},\text{OSUB})$ in the setting with $K = 17$ and $K = 129$ and a line graph.}

\begin{scriptsize}
\begin{tabular}{|r|cc|}
\cline{2-3}
\multicolumn{1}{c|}{} & \multicolumn{2}{c|}{$K$}\\
\cline{2-3}
\multicolumn{1}{c|}{} & $17$ & $129$\\
\hline
KLUCB & $3.08 \pm 0.05$ & $6.51 \pm 0.07$ \\
TS & $1.34 \pm 0.07$ & $2.68 \pm 0.05$ \\
UTS & $\mathbf{0.52 \pm 0.07}$ & $\mathbf{0.76 \pm 0.15}$ \\
\hline
\end{tabular}
\end{scriptsize}
\label{tab:res_line}
\end{table}

\section{Experimental Evaluation}

In this section, we compare the empirical performance of the proposed algorithm \NAME{} with the performance of a number of algorithms. We study the performance of the state--of--the--art algorithm OSUB~\cite{combes2014unimodala} to evaluate the improvement due to the employment of Bayesian approaches w.r.t. frequentist approaches. Furthermore, we study the performance of TS~\cite{thompson1933likelihood} to evaluate the improvement in Bayesian approaches due to the exploitation of the problem structure. For completeness, we study also the performance of  KLUCB~\cite{garivier2011kl}, being a frequentist algorithm that is optimal for Bernoulli distributions.

\subsubsection{Figures of merit} Given a policy $\mathfrak{U}$, we evaluate the average and 95\%--confidence intervals of the following figures of merit:
\begin{itemize}
	\item the pseudo--regret $R_T(\mathfrak{U})$ as defined in Equation~\ref{pseudoregretdefinition}; the lower $R_T(\mathfrak{U})$ the better the performance; 
	\item the regret ratio $R_{\%}(\mathfrak{U}_1, \mathfrak{U}_2)=\frac{R_T (\mathfrak{U}_1)}{R_T (\mathfrak{U}_2)}$ showing the ratio between the total regret of policy $\mathfrak{U}_1$ after $T$ rounds and the one obtained with $\mathfrak{U}_2$; the lower $R_{\%}(\mathfrak{U}_1,\mathfrak{U}_2)$ the larger the relative improvement of $\mathfrak{U}_1$ w.r.t. $\mathfrak{U}_2$.	
\end{itemize}

\paragraph{Line graphs} We initially consider the same experimental settings, composed of line graphs, that are studied in~\cite{combes2014unimodala}. They consider graphs with $K \in \lbrace 17, 129 \rbrace$ arms, where the arms are ordered on a line from the arm with smallest index to the arm with the largest index and with Bernoulli rewards whose averages have a triangular shape with the maximum on the arm in the middle of the line. More precisely, the minimum average is $0.1$, associated with arms $a_1$ and $a_{17}$ when $K = 17$ and with arms $a_1$ and $a_{129}$ with $K = 129$, while the maximum average reward is $\mu^* = 0.9$, associated with arm $a_9$ when $K = 17$ and with arm $a_{65}$ with $K = 129$. The averages decrease linearly from the maximum one to the minimum one.

For both the experiments, we average the regret over $100$ independent trials of length $T = 10^5$. We report  $R_t(\mathfrak{U})$ for each policy $\mathfrak{U}$ as~$t$ varies in Fig.~\ref{img:res_line}(a), for $K = 17$, and in Fig.~\ref{img:res_line}(b), for $K = 129$. The \NAME{} algorithm outperforms all the other algorithms along the whole time horizon, providing a significant improvement in terms of regret w.r.t.~the state--of--the--art algorithms. In order to have a more precise evaluation of the reduction of the regret w.r.t.~OSUB algorithm, we report $R_{\%}(\mathfrak{U},\text{OSUB})$ in Tab.~\ref{tab:res_line}. As also confirmed below by a more exhaustive series of experiments, in line graphs the relative improvement of performance due to \NAME{} w.r.t.~OSUB reduces as the number of arms increases, while the relative improvement of performance due to \NAME{} w.r.t.~TS increases as the number of arms increases. 

\begin{figure}[ht!]		
		\centering
		\subfloat[]{
				\hspace{-0.25cm}
				\includegraphics[scale=0.68]{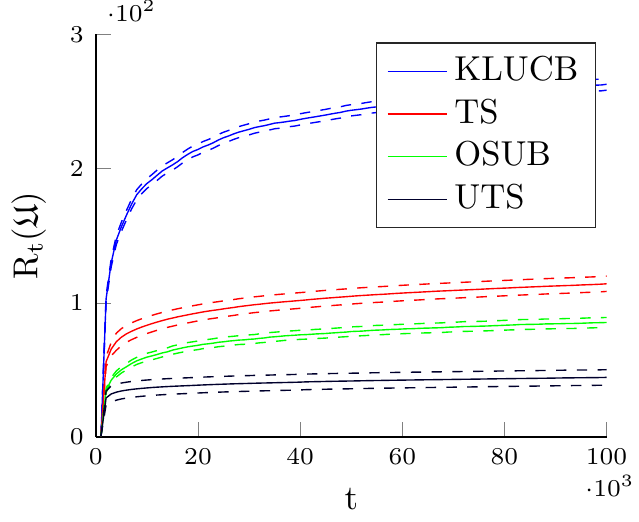}
		}
		\subfloat[]{
				\hspace{-0.4cm}
				\includegraphics[scale=0.68]{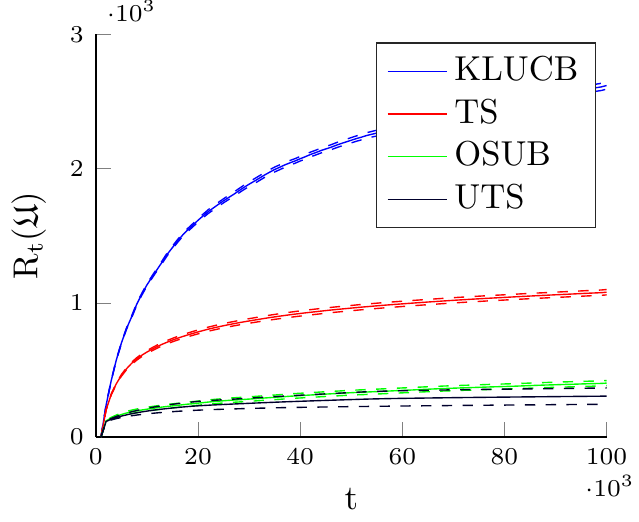}
		}
				\caption{Results for the pseudo--regret $R_t(\mathfrak{U})$ in line graphs settings with $K = 17$ (a) and $K = 129$ (b) as defined in~\cite{combes2014unimodala}.}
\label{img:res_line}
\end{figure}

\subsubsection{Erd\H{o}s-R\'enyi graphs} To provide a thorough experimental evaluation of the considered algorithms in settings in which the space of arms has a graph structure, we generate graphs using the model proposed by Erd\H{o}s and R\'enyi~\shortcite{erdds1959random}, which allows us to simulate graph structures more complex than a simple line. An Erd\H{o}s-R\'enyi graph is generated by connecting nodes randomly: each edge is included in the graph with probability $p$, independently from existing edges. We consider connected graphs with $K \in \lbrace 5, 10, 20, 50, 100, 1000 \rbrace$ and with probability $p \in \lbrace 1, \frac{1}{2}, \frac{\log(K)}{K}, \ell \rbrace$, where $p = 1$ corresponds to have a fully connected graph and therefore the graph structure is useless, $p = \frac{1}{2}$ corresponds to have a number of edges that increases linearly in the number of nodes, $p=\frac{\log(K)}{K}$ corresponds to have a few edges w.r.t. the nodes, and we use $p=\ell$ to denote line graphs (these line graphs differ from those used for the experimental evaluation discussed above for the reward function, as discussed in what follows). We use different values of $p$ in order to see how the performance of \NAME{} changes w.r.t.~the number of edges in the graph; we remark that such an analysis is unexplored in the literature so far. The optimal arm is chosen randomly among the existing arms and its reward is given by a Bernoulli distribution with expected value $0.9$. The rewards of the suboptimal arms are given by Bernoulli distributions with expected value depending on their distance from the optimal one. More precisely, let $d^*_i$ be the shortest path from the $i$--th arm to the optimal arm and let:
\begin{equation*}
	d^*_{\max} = \max_{i \in \lbrace 1, \dots, K \rbrace} d^*_i
\end{equation*} 
be the maximum shortest path of the graph. The expected reward of the $i$--th arm is:
\begin{equation*}
	\mu_i = 0.9 - d^*_i \dfrac{(0.9 - 0.1)}{d^*_{\max}},
\end{equation*}
i.e., the arm with $d^*_{\max}$ has a value equal to $0.1$ and the expected rewards of the arms along the path from it to the optimal arm are evenly spaced between $0.1$ and $0.9$. We generate $10$ different graphs for each combination of $K$ and $p$ and we run $100$ independent trials of length $T = 10^5$ for each graph. We average the regret over the results of the $10$ graphs. 

In Tab.~\ref{tab:res_graphs}, we report $R_T(\mathfrak{U})$ for each combination of policy $\mathfrak{U}$, $K$, and $p$. It can be observed that the \NAME{} algorithm outperforms all the other algorithms, providing in every case the smallest regret except for $K=1000$ and $p = \ell$. Below we discuss how the relative performance of the algorithms vary as the values of the parameters $K$ and $p$ vary.

\begin{table}[t!]
		\centering
\caption{Results concerning $R_T(\mathfrak{U})$ ($T = 10^5$) in the setting with Erd\H{o}s-R\'enyi graphs.}
{
\fontsize{6}{5}
\setlength{\tabcolsep}{2pt}  
\begin{tabular}{|c|c|c|cccc|}
\cline{4-7}
\multicolumn{3}{c|}{} & \multicolumn{4}{c|}{$p$}\\
\cline{4-7}
\multicolumn{3}{c|}{} & $1$ & $1/2$ & $\log(K)/K$ & $\ell$ \\
\hline
\multirow{24}{*}{\rotatebox[origin=c]{90}{$K$}} & \multirow{4}{*}{\rotatebox[origin=c]{90}{$5$}} & \scriptsize{KLUCB} & $34 \pm 0.4$ & $50 \pm 1.5$ & $52 \pm 3.7$ & $56 \pm 2.2$ \\
& & \scriptsize{TS} & $18 \pm 0.2$ & $23 \pm 0.6$ & $24 \pm 1.3$ & $25 \pm 0.7$ \\
& & \scriptsize{OSUB} & $34 \pm 0.3$ & $32 \pm 7.2$ & $35 \pm 5.8$ & $31 \pm 4.1$ \\
& & \scriptsize{UTS} & \boldmath$17 \pm 0.1$ & \boldmath$15 \pm 2.4$ & \boldmath$16 \pm 2.2$ & \boldmath$14 \pm 1.3$ \\
\cline{2-7}
& \multirow{4}{*}{\rotatebox[origin=c]{90}{$10$}} & \scriptsize{KLUCB} & $77 \pm 0.5$ & $107 \pm 5.5$ & $127 \pm 11.2$ & $159 \pm 7.0$ \\
& & \scriptsize{TS} & $40 \pm 0.2$ & $50 \pm 2.0$ & $56 \pm 3.8$ & $67 \pm 2.5$ \\
& & \scriptsize{OSUB} & $77 \pm 0.3$ & $76 \pm 8.1$ & $57 \pm 5.6$ & $70 \pm 8.1$ \\
& & \scriptsize{UTS} & \boldmath$39 \pm 0.2$ & \boldmath$35 \pm 3.2$ & \boldmath$27 \pm 2.1$ & \boldmath$34 \pm 2.4$ \\
\cline{2-7}
& \multirow{4}{*}{\rotatebox[origin=c]{90}{$20$}} & \scriptsize{KLUCB} & $163 \pm 0.7$ & $217 \pm 6.2$ & $262 \pm 16.2$ & $386 \pm 21.3$ \\
& & \scriptsize{TS} & $84 \pm 0.5$ & $102 \pm 2.3$ & $117 \pm 5.7$ & $157 \pm 6.9$ \\
& & \scriptsize{OSUB} & $163 \pm 0.8$ & $148 \pm 14.9$ & $86 \pm 14.6$ & $124 \pm 11.7$ \\
& & \scriptsize{UTS} & \boldmath$83 \pm 0.3$ & \boldmath$70 \pm 6.0$ & \boldmath$44 \pm 4.8$ & \boldmath$65 \pm 8.8$ \\
\cline{2-7}
& \multirow{4}{*}{\rotatebox[origin=c]{90}{$50$}} & \scriptsize{KLUCB} & $420 \pm 0.7$ & $560 \pm 15.0$ & $686 \pm 30.5$ & $1132 \pm 49.2$ \\
& & \scriptsize{TS} & $217 \pm 0.5$ & $262 \pm 4.4$ & $303 \pm 10.0$ & $454 \pm 19.9$ \\
& & \scriptsize{OSUB} & $420 \pm 1.0$ & $382 \pm 35.6$ & $162 \pm 13.9$ & $240 \pm 15.8$ \\
& & \scriptsize{UTS} & \boldmath$216 \pm 0.7$ & \boldmath$182 \pm 14.2$ & \boldmath$89 \pm 5.5$ & \boldmath$156 \pm 30.1$ \\
\cline{2-7}
& \multirow{4}{*}{\rotatebox[origin=c]{90}{$100$}} & \scriptsize{KLUCB} & $846 \pm 2.0$ & $1134 \pm 17.8$ & $1313 \pm 59.7$ & $2327 \pm 63.5$ \\
& & \scriptsize{TS} & \boldmath$436 \pm 1.1$ & $528 \pm 4.9$ & $586 \pm 18.4$ & $973 \pm 31.8$ \\
& & \scriptsize{OSUB} & $846 \pm 2.7$ & $786 \pm 39.0$ & $226 \pm 27.1$ & $369 \pm 10.7$ \\
& & \scriptsize{UTS} & \boldmath$437 \pm 0.5$ & \boldmath$372 \pm 15.2$ & \boldmath$141 \pm 9.1$ & \boldmath$290 \pm 42.3$ \\
\cline{2-7}
& \multirow{4}{*}{\rotatebox[origin=c]{90}{$1000$}} & \scriptsize{KLUCB} & $8505 \pm 12.2$ & $11247 \pm 60.1$ & $12024 \pm 464.7$ & $10640 \pm 291.5$ \\
& & \scriptsize{TS} & \boldmath$4391 \pm 3.4$ & $5262 \pm 23.0$ & $5478 \pm 151.3$ & $6554 \pm 115.2$ \\
& & \scriptsize{OSUB} & $8493 \pm 13.6$ & $7761 \pm 153.4$ & $1151 \pm 45.0$ & \boldmath$1165 \pm 20.7$ \\
& & \scriptsize{UTS} & \boldmath$4388 \pm 5.2$ & \boldmath$3718 \pm 62.9$ & \boldmath$1000 \pm 14.2$ & \boldmath$1165 \pm 41.8$ \\
\hline
\end{tabular}
}
\label{tab:res_graphs}
\end{table}

\emph{Consider the case with $p = 1$.} The performance of \NAME{} and TS are approximately equal and the same holds for the performance of OSUB and KLUCB. This is due to the fact that the neighborhood of each node is composed by all the arms, the graphs being fully connected, and therefore \NAME{} and OSUB cannot take any advantage from the structure of the problem. We notice, however, that \NAME{} and TS have not the same behavior and that \NAME{} always performs slightly better than TS. It can be observed in Fig.~\ref{img:res_p1} with $K = 5$ and $p = 1$ that the relative improvement is mainly at the beginning of the time horizon and that it goes to zero as $K$ increases (the same holds for OSUB w.r.t.~KLUCB). The reason behind this behavior is that \NAME{} reduces the exploration performed by TS in the first rounds, forcing the algorithm to pull the leader---chosen as the arm maximizing the empirical mean---for a larger number of rounds.

\begin{figure}[ht!]		
		\centering
		\includegraphics[scale=0.68]{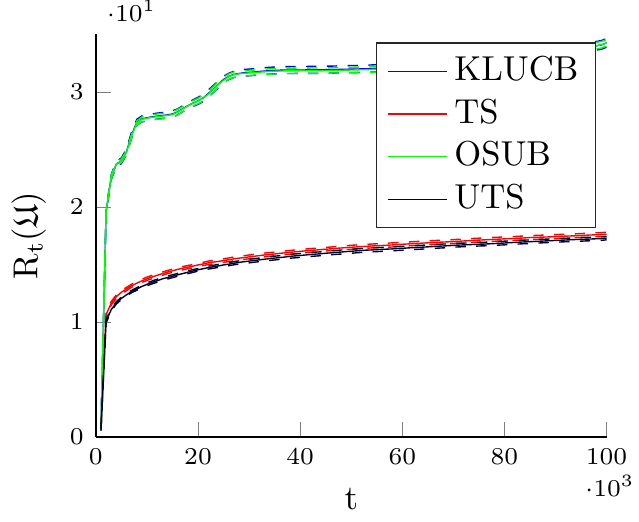}
		\caption{Results for the pseudo--regret $R_t(\mathfrak{U})$ in the setting with $K = 5$ and $p = 1$.}
\label{img:res_p1}
\end{figure}

\emph{Consider the case with $p = \frac{1}{2}$.} In the considered experimental setting, the relative performance of the algorithms does not depend on $K$. The ordering, from the best to the worst, over the performance of the algorithms is: \NAME{}, TS, OSUB, and finally KLUCB. Surprisingly, even the dependency of the following ratios on $K$ is negligible: $R_{\%}(\text{UTS}, \text{TS}) = 0.68 \pm 0.03$, $R_{\%}(\text{UTS}, \text{OSUB}) = 0.47 \pm 0.01$, and $R_{\%}(\text{OSUB}, \text{KLUCB}) = 0.68 \pm 0.03$. This shows that the relative improvement due to \NAME{} is constant w.r.t. TS and OSUB as $K$ varies. These results raise the question whether the relative performance of OSUB and TS would be the same, except for the numerical values, for every $p$ constant w.r.t. $K$. To answer to this question, we consider the case in which $p=0.1$, corresponding to the case in which the number of edges is linear in $K$, but it is smaller than the case with $p=\frac{1}{2}$. The results in terms of $R_T(\mathfrak{U})$, reported in Table~\ref{tab:p01} show that OSUB outperforms TS for $K\geq 10$, suggesting that, when $p$ is constant in $K$, OSUB may or may not outperform TS depending on the specific pair $(p,K)$.

\begin{table}[b]
\centering
\caption{Results concerning $R_T(\mathfrak{U})$ ($T = 10^5$) in the setting with Erd\H{o}s-R\'enyi graphs and $p = 0.1$.}
\label{tab:p01}
\begin{tabular}{|r|cccccc|}
\cline{2-7}
\multicolumn{1}{c|}{}		&\multicolumn{6}{c|}{$K$}		\\ \cline{2-7}
\multicolumn{1}{c|}{}		&	$5$		&	$10$		&	$20$		&	$50$			&  $100$	&	$1000$		\\	\hline
TS    	&	$\mathbf{25}$ 	& 	$66$		&  	$162$	&  	$278$		&  $519$	&  	$4564$		\\	
OSUB  	&	$29$  	&	$\mathbf{64}$  	&	$\mathbf{126}$	&  	$\mathbf{144}$		&  $\mathbf{266}$  	&	$\mathbf{2358}$		\\	\hline
\end{tabular}
\end{table}

\emph{Consider the case with $p = \frac{\log(K)}{K}$.}  The ordering over the performance of the algorithms changes as $K$ varies. More precisely, while \NAME{} keeps to be the best algorithm for every $K$ and KLUCB the worst algorithm for every $K$, the ordering between TS and OSUB changes. When $K \leq 10$ TS performs better than OSUB, instead when $K \geq 20$ OSUB outperforms TS, see Fig.~\ref{img:res_p_log}. This is due to the fact that, with a  small number of arms, exploiting the graph structure is not sufficient for a frequentist algorithm to outperform the performance of TS, while with many arms exploiting the graph structure even with a frequentist algorithm is much better than employing a general-purpose Bayesian algorithm. The ratio $R_{\%}(\text{UTS},\text{TS})$ monotonically decreases as $K$ increases, from $0.66$ when $K=5$ to $0.19$ when $K=1000$, suggesting that exploiting the graph structure provides advantages as $K$ increases. Instead, the ratio $R_{\%}(\text{UTS},\text{OSUB})$ monotonically increases as $K$ increases, from $0.45$ when $K=5$ to $0.94$ when $K=1000$, suggesting that the improvement provided by employing Bayesian approaches reduces as $K$ increases as observed above in line graphs.

\emph{Consider the case with $p = \ell$.} As in the case discussed above, OSUB is outperformed by TS for a small number of arms ($K \leq 10$), while it outperforms TS for many arms ($K \geq 20$). The reason is the same above. Similarly, the ratio $R_{\%}(\text{UTS},\text{TS})$ monotonically decreases as $K$ increases, from $0.58$ when $K = 5$ to $0.18$ when $K = 1000$, and the ratio $R_{\%}(\text{UTS},\text{OSUB})$ monotonically increases as $K$ increases, from $0.45$ when $K = 5$ to  $1.00$ when $K = 1000$. This confirms that the performance of UTS and the one of OSUB asymptotically match as $K$ increases when $p = \ell$ (as well as $p=\frac{\log(K)}{K}$). In order to investigate the reasons behind such a behavior, we produce an additional experiment with the line graphs of Combes and Proutiere~\shortcite{combes2014unimodala} except that the maximum expected reward is set to $0.108$ when $K = 17$ and $0.165$ when $K = 129$ (thus, given any edge with terminals $i$ and $i+1$, we have $|\mu_i -\mu_{i+1}|= 0.001$). What we observe (details of these experiments and those described below are in the appendices) is that, on average, OSUB outperforms \NAME{}  at $T=10^5$  suggesting that, when it is necessary to repeatedly distinguish between three arms that have very similar expected rewards, frequentist methods may outperform the Bayesian ones. This is no longer true when $T$ is much larger, e.g., $T=10^7$, where UTS outperforms OSUB (interestingly, differently from what happens in the other topologies, in line graphs with very small $|\mu_i -\mu_{i+1}|$, the average $R_T(\text{UTS})$ and $R_T(\text{OSUB})$ cross a number of times during the time horizon). Futhermore, we evaluate how the relative performance of OSUB w.r.t.~\NAME{} varies for $|\mu_i -\mu_{i+1}| \in \{0.001, 0.002, 0.005\}$, observing it improves as $|\mu_i -\mu_{i+1}|$ decreases. Finally, we evaluate whether this behavior emerges also in Erd\H{o}s-R\'enyi  graphs in which $p = \frac{c}{K}$ where $c$ is a constant (we use $p=\frac{5}{K},\frac{10}{K}$) and we observe that \NAME{} outperforms OSUB, suggesting that line graphs with very small $|\mu_i -\mu_{i+1}|$ are pathological instances for \NAME{}.

\begin{figure}[t!]	
				\centering
		\subfloat[]{
				\hspace{-0.25cm}
				\includegraphics[scale=0.68]{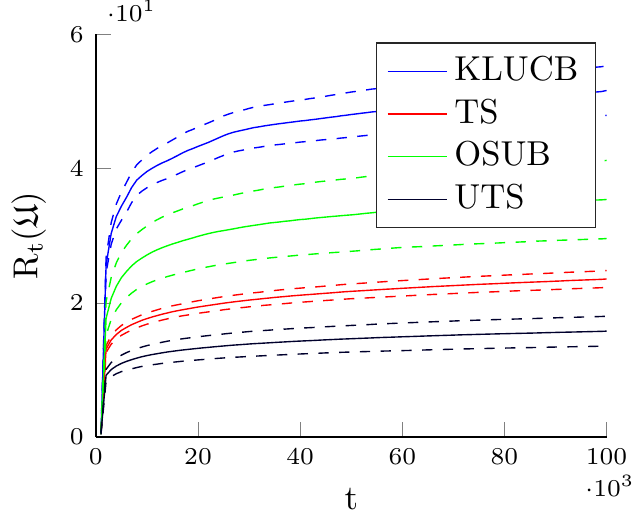}
		}
		\subfloat[]{
				\hspace{-0.4cm}
				\includegraphics[scale=0.68]{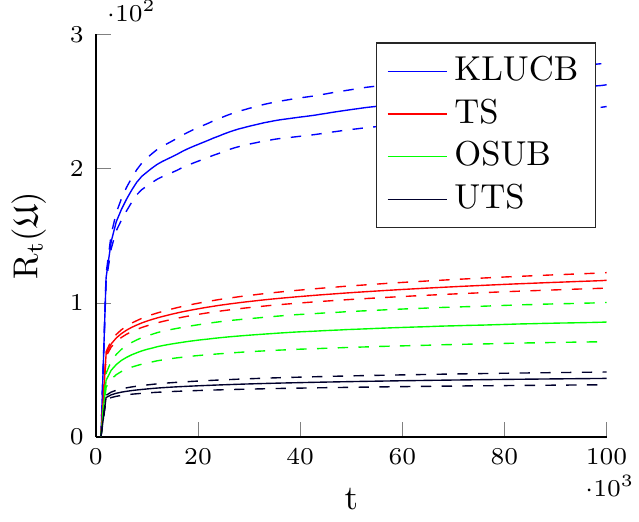}
		}
				\caption{Results for the pseudo--regret $R_t(\mathfrak{U})$ in the setting with $K = 5$ (a) and  $K = 20$ (b) and $p = \frac{\log(K)}{K}$.}
		\label{img:res_p_log}
\end{figure}

\section{Conclusions and Future Work}

In this paper, we focus on the Unimodal Multi--Armed Bandit problem with graph structure in which each arm corresponds to a node of a graph and each edge is associated with a relationship in terms of expected reward between its arms. We propose, to the best of our knowledge, the first Bayesian algorithm for the UMAB setting, called \NAME{}, which is based on the well--known Thompson Sampling algorithm. We derive a tight upper bound for \NAME{} that asymptotically matches the lower bound for the UMAB setting, providing a non-trivial derivation of the bound. Furthermore, we present a thorough experimental analysis showing that our algorithm outperforms the state--of--the--art methods.

In future, we will evaluate the performance of the algorithms considered in this paper with other classes of graphs, e.g., Barab\'{a}si--Albert and lattices. Future development of this work may consider an analysis of the proposed algorithm in the case of time--varying environments, i.e., the expected reward of each arm varies over time, assuming that the unimodal structure is preserved. Another interesting study may consider the case of a continuous decision space.

\bibliographystyle{aaai}
\bibliography{references}

\clearpage
\appendix
\onecolumn
\subsection{Appendix A: Proof of Theorem~\ref{teo:uts}}
\teouts*

\begin{proof}
At first, the regret of the \NAME{} algorithm $R_T(\text{\NAME{}})$ can be rewritten by dividing the $T$ rounds in two sets: those rounds in which the best arm $a^*$ is the leader , i.e., $l(t) = i^*$, and those in which the leader is another arm, i.e., $l(t) \neq i^*$:
\begin{align*}
	R_T(\text{\NAME{}}) &= \sum_{i \neq i^*} (\mu^* - \mu_i) \mathbb{E} [T_{i,T}]\\
	& = \sum_{i \neq i^*} (\mu^* - \mu_i) \mathbb{E} \left[ \sum_{t=1}^{T} \mathbf{1} \lbrace i_t = i \rbrace \right]\\
	&= \underbrace{\sum_{i \neq i^*} (\mu^* - \mu_i) \mathbb{E} \left[ \sum_{t=1}^{T} \mathbf{1} \lbrace l(t) = i^*\wedge i_t = i \rbrace \right]}_{\mathcal{R}_1} + \\
	&+ \underbrace{\sum_{i \neq i^*} (\mu^* - \mu_i) \mathbb{E} \left[ \sum_{t=1}^{T} \mathbf{1} \lbrace l(t) \neq i^*\wedge i_t = i \rbrace \right]}_{\mathcal{R}_2}\\
\end{align*}
Let us focus on $\mathcal{R}_1$. When $i^*$ is the leader, the proposed algorithm behaves like Thompson Sampling restricted to the optimal arm and its neighborhood $N^+(i^*)$, and the regret upper bound is the one presented in Theorem~1 in~\cite{kaufmann2012thompson} for TS algorithm, i.e., for every $\eps > 0 $:
\begin{equation}\label{equ:tsbound}
	\mathcal{R}_1 \leq (1 + \eps) \sum_{i \in N(i^*)} \frac{\mu^* - \mu_i}{KL(\mu_i,\mu^*)} [\log(T) + \log\log(T)] + C_1,
\end{equation}
where $C_1$ is an appropriate constant depending on $\eps$ and on the expected rewards $\mu_i$ of arms in $N^+(i^*)$.

Now let us consider $\mathcal{R}_2$, we have:
\begin{align*}
	\mathcal{R}_2 &= \sum\limits_{i \neq i^*} \underbrace{(\mu^* - \mu_i)}_{ \leq 1} \mathbb{E} \left[ \sum_{t = 1}^{T} \mathbf{1} \lbrace l(t) \neq i^*\wedge i_t = i \rbrace \right] \\
	& \leq \sum\limits_{i \neq i^*} \mathbb{E} \left[ L_{i,T} \right].
\end{align*}
Here we want to upper bound the number of times $a_i$ has been the leader $L_{i,T}$ with $\hat{L}_{i,T}$ defined as the number of rounds spent with $a_i$ as leader in the case only its neighborhood is considered during the whole time horizon $T$. This is clearly an upper bound over $L_{i,T}$, since there is nonzero probability that the \NAME{} algorithms moves in another neighborhood. From now on in the proof the analysis is carried on an algorithm working only on a unique neighborhood $N(i)$. 
\begin{align*}
	\mathcal{R}_2 & \leq \sum\limits_{i \neq i^*} \mathbb{E} \left[ L_{i,T} \right] \leq \sum\limits_{i \neq i^*} \mathbb{E} \left[ \hat{L}_{i,T} \right] = \sum\limits_{i \neq i^*} \sum_{t = 1}^T \mathbb{E} \left[ \mathbf{1} \lbrace l(t) = i \rbrace \right] \\
	& = \sum\limits_{i \neq i^*} \sum_{t = 1}^T \mathbb{E} \left[ \mathbf{1} \lbrace \hat{\mu}_{i,t} = \max_{a_j \in N(i)} \hat{\mu}_{j,t} \rbrace \right], \\
\end{align*}
where, with abuse of notation, $l(t)$ is the leader at round $t$ in this new problem where only $N(i)$ is considered.

When $i \neq i^*$ is the leader, $a_{i}$ is not the optimal arm. Thus, since we are in a unimodal setting, it exists an optimal arm $a_{i'} \in N(i), i' \neq i$ s.t.~$\mu_{i'} = \max_{i | a_i \in N(i)} \mu_i$. Nonetheless, since $a_{i}$ is the leader, its empirical mean is the maximum in its neighborhood and, in particular, $\hat{\mu}_{i,t} \geq \hat{\mu}_{i'}$. Thus, we have:
\begin{align*}
	\mathcal{R}_2 & \leq \sum\limits_{i \neq i^*} \sum_{t = 1}^T \mathbb{E} \left[ \mathbf{1} \lbrace \hat{\mu}_{i,t} = \max_{a_j \in N(i)} \hat{\mu}_{j,t} \rbrace \right] \\
	& \leq \sum\limits_{i \neq i^*} \sum_{t = 1}^T \mathbb{E} \left[ \mathbf{1} \lbrace \hat{\mu}_{i,t} \geq \hat{\mu}_{i',t} \rbrace \right] \\
	& = \sum\limits_{i \neq i^*} \sum_{t = 1}^T \mathbb{P} \left( \hat{\mu}_{i,t} \geq \hat{\mu}_{i',t} \right) \\
	& = \sum\limits_{i \neq i^*} \sum_{t = 1}^T \mathbb{P} \left( \hat{\mu}_{i,t} - \mu_i - \frac{\Delta_i}{2} - \hat{\mu}_{i',t} + \mu_{i'} - \frac{\Delta_i}{2} \geq 0 \right) \\
	& \leq \sum\limits_{i \neq i^*} \left[ \underbrace{\sum_{t = 1}^T \mathbb{P} \left( \hat{\mu}_{i,t} - \mu_i - \frac{\Delta_i}{2} \geq 0 \right)}_{\mathcal {R}_{i1}} + \underbrace{\sum_{t = 1}^T \mathbb{P} \left( \hat{\mu}_{i',t} - \mu_{i'} + \frac{\Delta_i}{2} \leq 0 \right)}_{\mathcal {R}_{i2}} \right], \\
\end{align*}
where $\Delta_i = \max_{i' |a_i \in N(i)} \mu_{i'} - \mu_i$ denotes the expected loss incurred in choosing arm $a_i$ instead of its best adjacent one $a_{i'}$.

Let us focus on $\mathcal{R}_{i1}$:
\begin{align*}
	\mathcal{R}_{i1} & = \sum_{t = 1}^T \mathbb{P} \left( \hat{\mu}_{i,t} \geq \mu_i + \frac{\Delta_i}{2} \right) \\
	& = \sum_{t = 1}^T \sum_{h = 1}^t \mathbb{P} \left( T_{i,t} = h \wedge \hat{\mu}_{i,t} \geq \mu_i + \frac{\Delta_i}{2} \right) \\
	& = \sum_{t = 1}^T \sum_{h = 1}^t \mathbb{P} \left( T_{i,t} = h \ \vert \ \hat{\mu}_{i,t} \geq \mu_i + \frac{\Delta_i}{2}\right) \mathbb{P} \left( \hat{\mu}_{i,t} \geq \mu_i + \frac{\Delta_i}{2} \right)\\
	& \leq \sum_{t = 1}^T \sum_{h = 1}^t \mathbb{P} \left( T_{i,t} = h \ \vert \ \hat{\mu}_{i,t} \geq \mu_i + \frac{\Delta_i}{2}\right) e^{-\frac{h \Delta_i^2}{2}}
\end{align*}
Where the last inequality is due to the Hoeffding inequality~\cite{hoeffding1963probability}. By relying on the fact that $\sum_{h = x+1}^\infty e^{-kh} \leq \frac{1}{k} e^{-kx}$ and by considering $x = \frac{t}{|N^+(i)|}$ we have:
\begin{align*}
	\mathcal{R}_{i1} & \leq \sum_{t = 1}^T \left( \sum_{h = 1}^{ \frac{t}{|N^+(i)|} } \underbrace{\mathbb{P} \left( T_{i,t} = h \ \vert \ \hat{\mu}_{i,t} \geq \mu_i + \frac{\Delta_i}{2}\right)}_{= 0} e^{\frac{h \Delta_i^2}{2}} + \frac{2}{\Delta_i^2} e^{ -\frac{ \frac{t}{|N^+(i)|} \Delta_i^2}{2}} \right) \\
	& = \sum_{t = 1}^T \frac{2}{\Delta_i^2} e^{\frac{ -\frac{t}{|N^+(i)|} \Delta_i^2}{2}} \leq C_2
\end{align*}
where $\mathbb{P} \left( T_{i,t} = h \ \vert \ \hat{\mu}_{i,t} \geq \mu_i + \frac{\Delta_i}{2}\right) = 0$ for $h \leq \frac{t}{|N^+(i)|}$ is due to the fact that the leader is chosen at least $\frac{t}{|N^+(i)|}$ over $t$ rounds and $C_2$ is a constant.

Let us focus on $\mathcal{R}_{i2}$ and the following proposition provided in~\cite{kaufmann2012thompson}:
\begin{proposition}
	If we use a TS policy over a set of finite arms $\{ a_i \}$ where $a_{i'}$ is the optimal one, there exist constants $b \in (0,1)$ and $C_b \leq \infty$ s.t.:
	\begin{equation}
		\sum_{t = 1}^\infty \mathbb{E} \left[ \mathbf{1} \{ T_{i',t} \leq t^b \} \right] \leq C_b.
	\end{equation}
\end{proposition}
Similarly to what has been derived for $\mathcal{R}_{i1}$ we have:
\begin{align*}
	\mathcal{R}_{i2} &= \sum_{t = 1}^T \mathbb{P} \left( \hat{\mu}_{i',s} \leq \mu_{i'} - \frac{\Delta_i}{2} \right)\\
	& = \sum_{t = 1}^T \sum_{h = 1}^t \mathbb{P} \left( T_{i',t} = h \wedge \hat{\mu}_{i',s} \leq \mu_{i'} - \frac{\Delta_i}{2}  \right) \\
	& = \sum_{t = 1}^T  \sum_{h = 1}^{t^b} \mathbb{P} \left( T_{i',t} = h \wedge \hat{\mu}_{i',s} \leq \mu_{i'} - \frac{\Delta_i}{2} \right) + \sum_{t = 1}^T \sum_{h = t^b+1}^t \underbrace{\mathbb{P} \left(T_{i',t} = h \ \vert \ \hat{\mu}_{i',s} \leq \mu_{i'} - \frac{\Delta_i}{2} \right)}_{\leq 1} \mathbb{P} \left( \hat{\mu}_{i',s} \leq \mu_{i'} - \frac{\Delta_i}{2} \right)\\
	& \leq \sum_{t = 1}^\infty \mathbb{E} \left[ \mathbf{1} \{ T_{i',t} \leq t^b \} \right] + \sum_{t = 1}^T \sum_{h = t^b+1}^t \mathbb{P} \left( \hat{\mu}_{i',s} \leq \mu_{i'} - \frac{\Delta_i}{2} \right)\\
	& \leq C_b + \sum_{t = 1}^T \sum_{h = t^b+1}^t e^{-\frac{t \Delta_i^2}{2}}\\
	& \leq C_b + \sum_{t = 1}^T \frac{2}{\Delta_i^2} e^{-\frac{t^b \Delta_i^2}{2}} \leq C_3\\
\end{align*}
since we are using TS in among arms in $N(i)$ and the last inequality holds for all $b \in (0,1)$.

By considering the three partial results on $\mathcal{R}_1,\mathcal{R}_{i1},\mathcal{R}_{i2}$ we have:
\begin{align*}
	& \mathcal{R}_T(\text{\NAME{}}) \leq \mathcal{R}_1 + \sum_{i \neq i^*} \left( \mathcal{R}_{i1} + \mathcal{R}_{i2} \right) \\
	&= (1 + \eps) \sum_{i \in N(i^*)} (\mu^* - \mu_i) \frac{\log(T) + \log\log(T)}{KL(\mu_i,\mu^*)} + C_1 + (K-1) (C_2 + C_3) 
\end{align*}
considering $\tilde{C} = C_1 + (K-1) (C_2 + C_3) $ concludes the proof.
\end{proof}

\clearpage
\subsection{Appendix B: Additional Results on $p = \ell$}

In order to investigate the reasons why the performance of \NAME{} and the one of OSUB asymptotically match as $K$ increases when $p=\ell$, we produce additional experiments with the line graphs described in~\cite{combes2014unimodala}. We generated line graphs where the minimum expected reward is set to $0.1$ and the maximum expected reward varies: given any edge with terminals two consecutive nodes $i$ and $i+1$, we generated graphs where $\Delta = |\mu_i -\mu_{i+1}| \in \lbrace 0.001, 0.002, 0.005 \rbrace$. More precisely, when $K = 17$, the expected reward of the central arm $a_8$ is set to, respectively, $0.108$, $0.116$ and $0.14$. When $K = 129$, the expected reward of the central arm $a_{65}$ is set to, respectively, $0.165$, $0.23$ and $0.425$. The results for $T = 10^7$ are reported in Figure~\ref{img:res_delta_k17} for $K = 17$ and in Figure~\ref{img:res_delta_k129} for $K = 129$.

\begin{figure}[ht!]		
		\centering
		\subfloat[]{
				\includegraphics[scale=0.9]{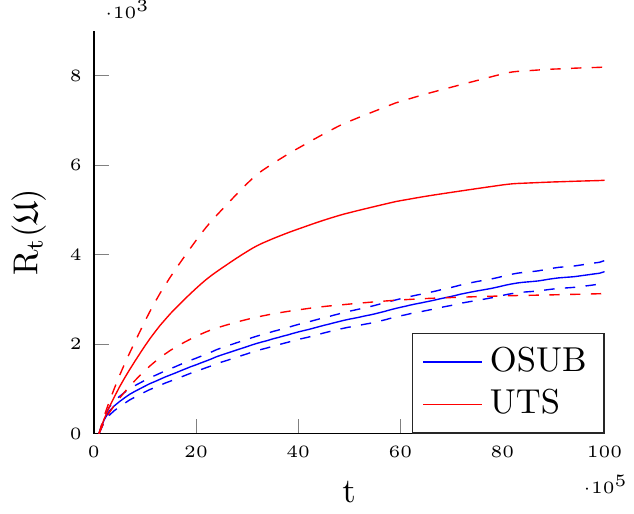}
		}
		\subfloat[]{
				\includegraphics[scale=0.9]{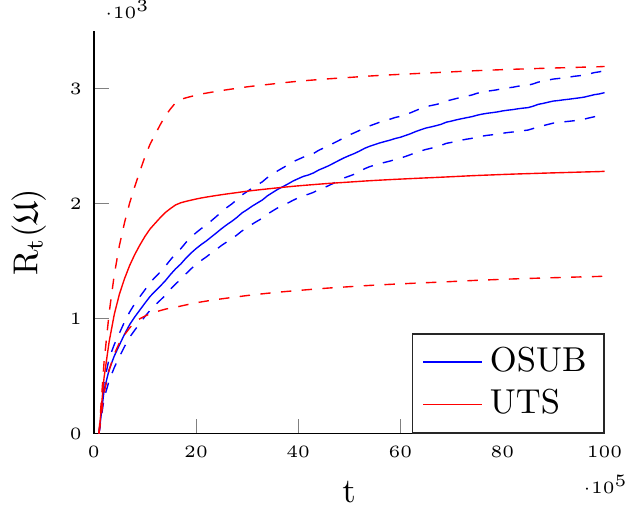}
		}
		\subfloat[]{
				\includegraphics[scale=0.9]{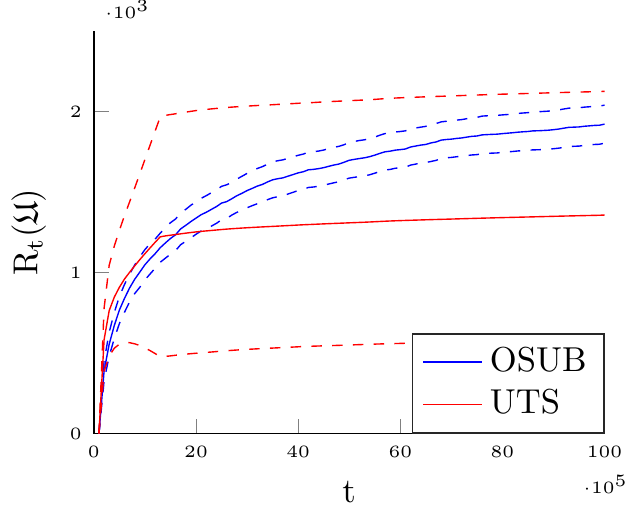}
		}
		\caption{Results for the pseudo--regret $R_t(\mathfrak{U})$ in the setting with $K = 17$, $p = \ell$ and $\Delta = 0.001$ (a), $\Delta = 0.002$ (b) and $\Delta = 0.005$ (c).}
		\label{img:res_delta_k17}
\end{figure}

\begin{figure}[ht!]		
		\centering
		\subfloat[]{
				\includegraphics[scale=0.9]{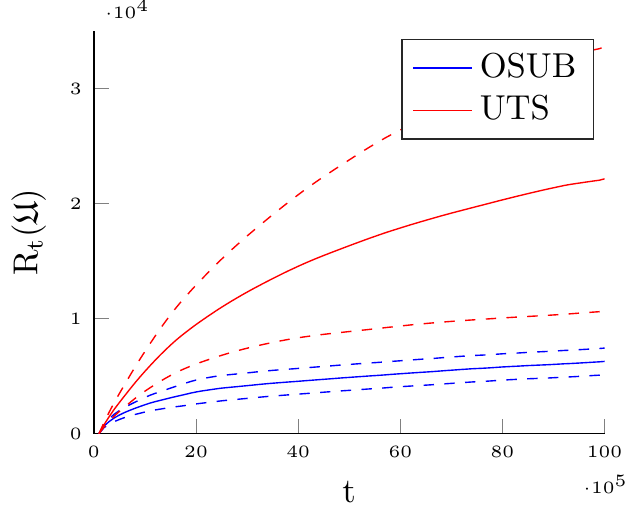}
		}
		\subfloat[]{
				\includegraphics[scale=0.9]{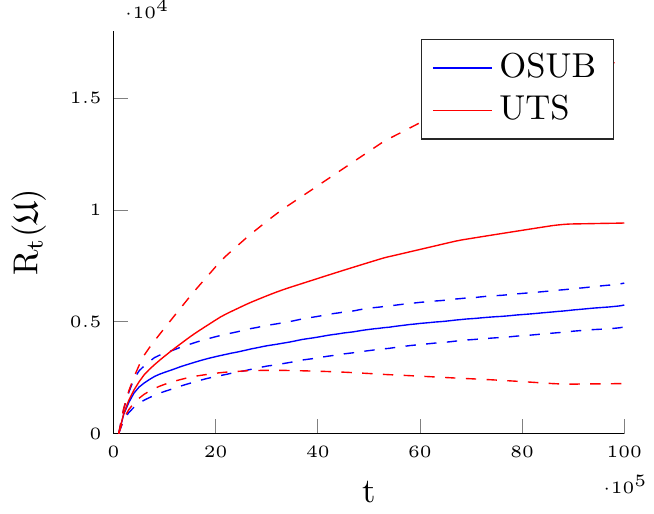}
		}
		\subfloat[]{
				\includegraphics[scale=0.9]{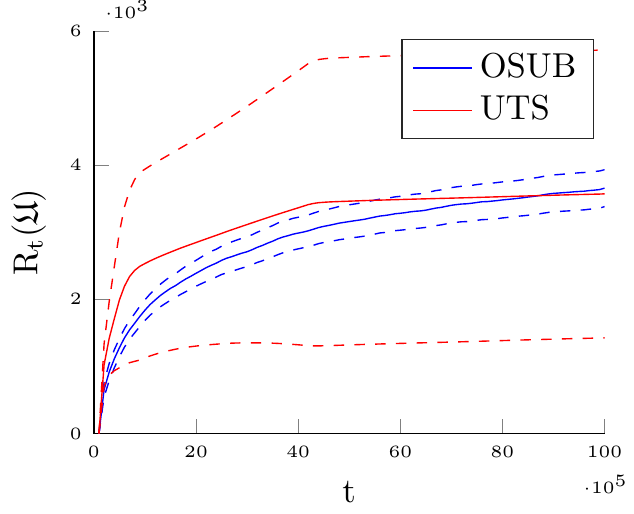}
		}
		\caption{Results for the pseudo--regret $R_t(\mathfrak{U})$ in the setting with $K = 129$, $p = \ell$ and $\Delta = 0.001$ (a), $\Delta = 0.002$ (b) and $\Delta = 0.005$ (c).}
		\label{img:res_delta_k129}
\end{figure}

We observe that, at $T = 10^7$ with $K = 17$ and $\Delta = 0.001$, on average OSUB outperforms \NAME{} while, with $\Delta \in \lbrace 0.002, 0.005 \rbrace$, at the end of the experiments \NAME{} outperforms OSUB. This behavior suggests that, even in the case with $\Delta = 0.001$, \NAME{} will perform better than OSUB for $T > 10^7$. In the case with $K = 129$, $\Delta = 0.005$ and $T = 10^7$, \NAME{} outperforms OSUB at the end of the experiments while with $\Delta \in \lbrace 0.001, 0.002 \rbrace$ OSUB performs better. Following the same line of reasoning, for $T > 10^7$ this could no longer be true.
All these results suggest that when it is necessary to repeatedly distinguish between three arms that have very similar expected rewards and very low expected rewards, frequentist methods may outperform the Bayesian ones at the beginning of the learning process, while Bayesian methods asymptotically outperform frequentist ones. In particular, we observe that the relative performance of OSUB w.r.t.~\NAME{} varies for $\Delta \in \{0.001,0.002,0.005\}$, observing it improves as $\Delta$ decreases.

Finally, we evaluate whether this behavior emerges also in Erd\H{o}s-R\'enyi  graphs in which $p = \frac{c}{K}$ where $c$ is a constant. We use $p \in \lbrace \frac{5}{K}, \frac{10}{K} \rbrace$ and $T = 10^6$. We observe that \NAME{} outperforms OSUB, suggesting that line graphs with very small $\Delta$ represent pathological instances for \NAME{}. The results are reported in Figure~\ref{img:res_c_K}.

\begin{figure}[ht!]		
		\centering
		\subfloat[]{
				\includegraphics[scale=0.9]{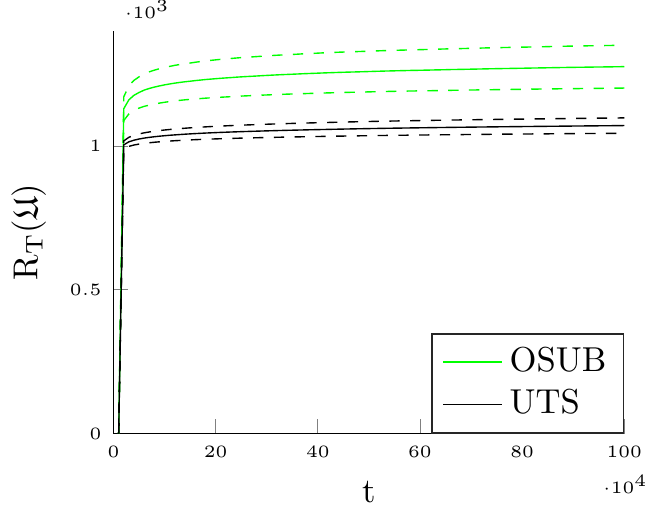}
		}
		\subfloat[]{
				\includegraphics[scale=0.9]{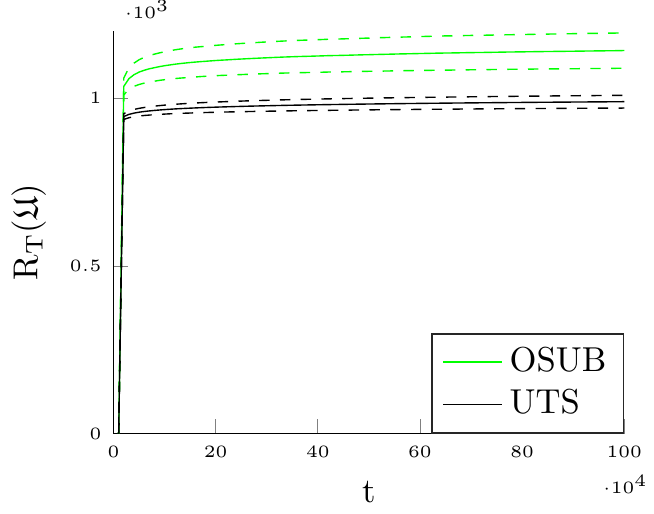}
		}
		\caption{Results for the pseudo--regret $R_t(\mathfrak{U})$ in the setting with $K = 1000$, $p = \frac{5}{K}$ (a) and $p = \frac{10}{K}$ (b).}
		\label{img:res_c_K}
\end{figure}

\end{document}